\documentclass[12pt]{article}

\usepackage{etoolbox}
\makeatletter
\patchcmd{\maketitle}
 {\def\@makefnmark}
 {\def\@makefnmark{}\def\useless@macro}
 {}{}
\makeatother

\usepackage{caption}
\usepackage{setspace}

\usepackage{hyperref}
\usepackage{subcaption}
\usepackage{numprint}
\usepackage{listings}
\usepackage{siunitx}

\usepackage{authblk}

\usepackage[round]{natbib}
\usepackage{graphicx}
\usepackage{textcomp}

\usepackage[utf8]{inputenc} % allow utf-8 input
\usepackage[T1]{fontenc}    % use 8-bit T1 fonts
\usepackage{hyperref}       % hyperlinks
\usepackage{url}            % simple URL typesetting
\usepackage{booktabs}       % professional-quality tables
\usepackage{amsfonts,amsthm}       % blackboard math symbols
\usepackage{nicefrac}       % compact symbols for 1/2, etc.
\usepackage{microtype}      % microtypography
\usepackage{color}
\usepackage{amsmath}
\usepackage[ruled,vlined,noresetcount,linesnumbered]{algorithm2e}

\usepackage{amsthm}
\usepackage{xcolor}

\theoremstyle{plain}
\newtheorem{thm}{Theorem}[section]

\theoremstyle{definition}
\newtheorem{definition}{Definition}[section]

\newcommand{\btheta}{\boldsymbol{\theta}}
\newcommand{\bxi}{\boldsymbol{\xi}}
\newcommand{\bmu}{\boldsymbol{\mu}}
\newcommand{\brho}{\boldsymbol{\rho}}
\newcommand{\boldeta}{\boldsymbol{\eta}}
\newcommand{\bg}{\boldsymbol{g}}
\newcommand{\bx}{\boldsymbol{x}}
\newcommand{\bI}{\boldsymbol{I}}
\newcommand{\bL}{\boldsymbol{L}}
\newcommand{\dataset}{\mathcal{D}}
\newcommand{\loss}{\mathcal{L}}
\newcommand{\diff}{\,\mathrm{d}}
\newcommand{\E}{\mathbb{E}}

\DeclareMathOperator{\im}{im}
\begin{document}

\title{Privacy-preserving Data Sharing on Vertically Partitioned Data}

\author{
  {Razane Tajeddine}$^1$, {Joonas J\"{a}lk\"{o}}$^2$,
		{Samuel Kaski}$^{2,3}$, and {Antti Honkela}$^{2}$ \\
	\vskip1ex
	$^1$ Helsinki Institute for Information Technology HIIT,\\
		Department of Computer Science, University of Helsinki, Finland \\
	$^2$ Helsinki Institute for Information Technology HIIT,\\
	Department of Computer Science, Aalto University, Finland\\
	$^3$ Department of Computer Science, University of Manchester
\thanks{This work was supported by the Academy of Finland (Flagship programme: Finnish Center for Artificial Intelligence, FCAI; and grants 325572, 325573, 343555), the Strategic Research Council at the Academy of Finland (Grant 336032), European Network of AI Excellence Centres ELISE (EU Horizon 2020 grant agreement 951847), as well as UKRI Turing AI World-Leading Researcher Fellowship, EP/W002973/1.\\
The authors wish to thank the Finnish Computing Competence Infrastructure (FCCI) for supporting this project with computational and data storage resources.}}

\maketitle

  %\runningtitle{Differentially Private Data Sharing on Vertically Partitioned Data}

  %\subtitle{...}

  \begin{abstract}
   {In this work, we introduce a differentially private method for generating synthetic data from vertically partitioned data, \emph{i.e.}, where data of the same individuals is distributed across multiple data holders or parties. We present a differentially privacy stochastic gradient descent (DP-SGD) algorithm to train a mixture model over such partitioned data using variational inference. We modify a secure multiparty computation (MPC) framework to combine MPC with differential privacy (DP), in order to use differentially private MPC effectively to learn a probabilistic generative model under DP on such vertically partitioned data.
   Assuming the mixture components contain no dependencies across different parties, the objective function can be factorized into a sum of products of the contributions calculated by the parties. Finally, MPC is used to compute the aggregate between the different contributions. Moreover, we rigorously define the privacy guarantees with respect to the different players in the system. To demonstrate the accuracy of our method, we run our algorithm on the Adult dataset from the UCI machine learning repository, where we obtain comparable results to the non-partitioned case.}
  \end{abstract}
  
%  \keywords{Differential Privacy, Multiparty Computation, Vertically partitioned, Mixture models, Variational Inference}

\section{Introduction}

Differential privacy (DP) \citep{dwork2006calibrating} provides a framework for developing algorithms that can use sensitive personal data while guaranteeing the privacy of the data subjects. One of the most interesting applications of DP is data anonymisation \cite{samarati1998generalizing,sweeney2002k} through creating an anonymised twin of a dataset. The anonymised dataset can then be used in place of the original for new analyses, while maintaining privacy for the original data subjects. In this paper, we investigate the problem of learning mixture models for vertically partitioned data under DP, as shown in Figure~\ref{idea}, where multiple parties hold different features for the same set of individuals. Privacy-preserving learning for vertically partitioned data would enable addressing completely new questions through combining for example health and shopping data that are held by different parties who cannot otherwise combine their data.

Creating synthetic data to ensure anonymity was first proposed by \citet{rubin1993statistical}. After the introduction of DP, several authors have proposed methods for releasing synthetic data protected by DP guarantees \citep[e.g.][]{blum2008learning,dwork2009complexity,
  xiao2010differentially,beimel2010bounds, chen2012differentially,
  hardt2012simple,zhang2014privbayes}.
Recent work has focused on applying generative machine learning models learned under DP using autoencoder neural networks \citep[e.g.][]{acs2018differentially}, generative adversarial networks (GANs) \citep[e.g.][]{yoon2018pategan}, discrete data generation using probabilistic graphical models \cite{zhang2014privbayes,mckenna2019pgm} or more general probabilistic models \citep{jalko2021privacy}.
DP provides a formal guarantee that the true data cannot be inferred from the generated synthetic data, something that would be very difficult to guarantee with other methods.

%Data sharing is beneficial for the scientific community, however, privacy concerns arise for sensitive data, such as, medical or financial history. For that reason, sharing a synthetic dataset is proposed by \cite{rubin1993statistical}. The distribution of the data in the synthetic dataset should be close to that of the original dataset, which could compromise the privacy of individuals. \emph{Differential privacy (DP)} provides privacy guarantees for synthetic data generation. Privacy-preserving data sharing using DP is discussed by \cite{jalko2021privacy} \joonas{There are plenty of other works discussing DP data sharing as well}.

%DP, which was introduced by \cite{dwork2006calibrating}, assumes that privacy is preserved if the outcome of an algorithm is barely affected whether a certain individual record is in the dataset or not. 
%More formally, an algorithm $A$ is \emph{$(\epsilon,\delta)$-differentially private}, for $\epsilon>0, \delta \in [0,1]$, if for any datasets $X$ and $X^\prime$ that differ by only a single record \[ P[A(X)\in S] \leq e^\epsilon P[A(X^\prime)\in S] + \delta. \]

%In many cases, the trained DP model would be more comrehensive using 
The data are often distributed across multiple parties. The most typical distributed data case is so-called horizontal partitioning, where different parties hold data for different individuals. This setting is the basis for \emph{federated learning} \citep{kairouz2019advances} which can be used to learn models by combining updates from multiple parties holding data for different individuals. Another setting is the vertical data partitioning, in which the data are divided such that each party holds different features of the same set of individuals. 

%Sharing data from multiple parties without privacy concerns would lead to better and more accurate models in machine learning, for instance, by combining large-scale health, financial, and behavioural data. %\rt{(Need to add references here)}
While the literature is broad on DP learning on horizontally partitioned data, there has only been very limited prior work on DP learning on vertically partitioned data. \citet{mohammed2013secure} proposed an algorithm for secure two-party DP data release for vertically partitioned data, but their algorithm is limited to discrete data with a small number of possible values because it requires enumerating all possible values of a data element. Recently \citet{wang2020hybrid} proposed a method for DP learning of generalised linear models on vertically partitioned data.

\emph{Secure multiparty computation} (MPC) \citep{yao1982protocols, maurer2006secure} is a cryptographic technique to allow multiple parties to compute a function of their private inputs without revealing their inputs.
Running the learning algorithm under MPC would in theory solve the privacy concern arising from vertical partitioning, but directly using MPC for large problems would be a bottleneck since it is computationally expensive. In this work, we limit the use of MPC to the final gradient aggregation which comprises of a few simple operations, in order to reduce the computational cost. We do this using a slightly modified version of CrypTen \citep{gunning2019crypten}, which adds MPC to PyTorch.

In addition to the challenge of securely combining the contributions of different parties, vertical data partitioning introduces the additional challenge of matching the records between different parties. For simplicity, we assume the records have been matched before running our algorithm. If necessary, the data can be matched first using a \emph{private entity resolution} algorithm \citep{getoor2012entity,sehili2015privacy}.

In this work, we propose an algorithm for DP data sharing by applying \emph{differentially private variational inference (DPVI)} and MPC to train a mixture model on vertically partitioned data, where we can privately infer the parameters of the distribution of the data. 
%We build the training model such that each party calculates the parts of the gradient based on its data.
The algorithm requires gradients of parameters based on the data, which we form such that each party computes only dimensions matching to its data. 
This information is combined using MPC, where the parties combine the gradients and use the typical DP-SGD operations of clipping and perturbing the gradients in every optimisation iteration. To the best of our knowledge, this is the first work to build a DP-SGD algorithm for vertically partitioned data. Particularly, we make the following contributions:

\begin{itemize}
    \item We provide an algorithm for training a mixture model on vertically partitioned data based on DP variational inference. This algorithm is generalizable to any model where the log likelihood does not have any direct dependencies among the parties.
    \item We rigorously define the privacy guarantees with respect to the different players in the system, \emph{i.e.}, the outside analyst and the parties.
    \item  We show that the performance of the algorithm is comparable to the non-partitioned DPVI algorithm by running both algorithms on the Adult dataset.
    \item The algorithm studied in this work requires high precision, which is not directly possible using the MPC framework Crypten. To solve this, we provide a method to modify the Crypten arithmetics for MPC to allow more precision in the fixed point representation. These additions could be easily used whenever a need for more precision arises.
\end{itemize}

The outline of the paper is as follows: In Section \ref{sec:prelim}, we describe the preliminaries needed for this work. Afterwards in Section \ref{sec:setup}, we will descibe the problem setup and theoretical analysis of the privacy protocol. In Section \ref{sec:model}, we describe the algorithm for training on vertically partitioned data. After defining the model, we state our privacy guarantees in Section \ref{sec:theorems}. In Section \ref{sec:exp}, we provide results from experiments on the Adult dataset \cite{Dua:2019} showing that the accuracy of the model is on-par with the non-partitioned case. Then we conclude with discussions and conclusions in Sections \ref{sec:disc} and \ref{sec:conc}.

\begin{figure}
\centering
\includegraphics[scale=0.55]{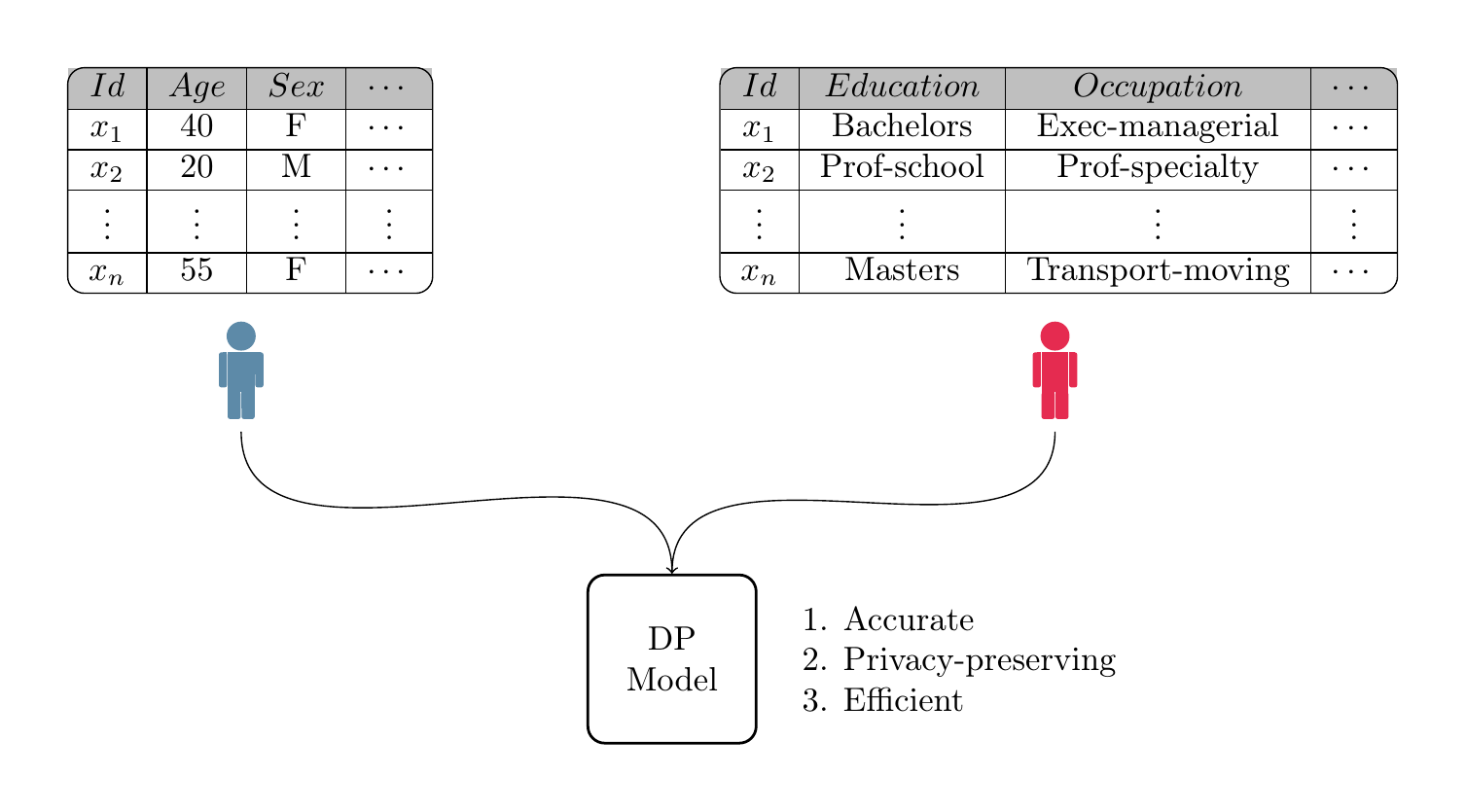}
\caption{Basic Idea of the problem: Multiple parties hold different sets of features for the same set of individuals. We want to train a DP model of the combined set that is accurate, efficient, and privacy-preserving, without the partners sharing their data with each other.}
\label{idea}
\end{figure}

%With the increasing use of predictive machine learning, the concern for user data privacy is increasing. To build predictive models, personal data, which is in many cases sensitive, is being collected and shared. Differential privacy is a privacy constraint which guarantees that the output of an algorithm on a dataset does not change much by adding or removing a record from the dataset.

%In many cases, the datasets are partitioned between multiple parties. Data could be partitioned arbitrarily, but the most common data partitioning are horizontal and vertical. Horizontal partitioning of data is the case such that each party holds the same set of features for different individuals/entities. On the other hand, vertical partitioning of the data is the case such that each party holds a different set of features for the same set of individuals. In this paper, we are concerned with vertically partitioned data.

%One of the commonly used machine learning methods is the principal component analysis (PCA). \cite{dwork2014analyze} constructed privacy preserving algorithms for PCA on non-partitioned data. In this work, we are concerned with constructing a privacy preserving algorithm for computing a low rank eigenspace on a vertically partitioned dataset. We assume that the data is matched across the multiple parties.

\section{Preliminaries}\label{sec:prelim}

\subsection{Differential Privacy (DP)}

Differential privacy (DP) \citep{dwork2006calibrating, dwork2014algorithmic} is a framework allowing learning significant information about a population, while learning almost nothing about any certain individual. DP provides a strong mathematical guarantee of privacy.

\begin{definition}
[($\epsilon,\delta$)-Differential Privacy] Let $\epsilon, \delta \geq 0$. A randomised algorithm $\mathcal{A}$ is ($\epsilon,\delta$)-differentially private if for all pairs of adjacent datasets $\dataset, \dataset'$, \emph{i.e.}, differing in one data sample, and all measurable subsets $S\subset \im(\mathcal{A})$
\[\Pr(\mathcal{A}(\dataset)\in S) \leq e^\epsilon \Pr(\mathcal{A}(\dataset')\in S)+\delta.\]
\end{definition}

The case $\delta=0$, is known as $\epsilon$-differentially privacy, or \emph{pure} DP. The parameters $\epsilon, \delta$ measure the strength of the privacy, where smaller values correspond to stronger privacy.

\subsubsection{DP stochastic gradient descent}
One of the most general and widely applied DP learning techniques is the \emph{DP stochastic gradient descent} (DP-SGD) \citep{Rajkumar2012,Song2013,Abadi2016}, where the parameters of a differentiable loss function $\loss(\btheta, \dataset) = \sum_i \loss_i(\btheta, \bx_i)$ are learned under DP. Here $\btheta$ are the parameters of the model and $\dataset = \{\bx_i\}_{i=1}^N$ is the dataset. For loss functions which can be represented as a sum over elements of a dataset, DP-SGD computes the gradients w.r.t.~the parameters for each individual, $\bg_i = \nabla_{\btheta} \loss_i(\btheta, \bx_i)$, on a randomly sampled subset of the data. The individual gradients $\bg_i$ are then projected into $B^d(0, C)$, a $d$-dimensional ball with radius $C$, where $d$ is the number of parameters, and finally summed and perturbed with spherical Gaussian noise with covariance matrix $C^2 \sigma^2 \bI_d$, where $\bI_d$ is the $d$-dimensional identity matrix and noise variance $\sigma^2$ controls the magnitude of added noise.

\subsubsection{Privacy accounting}

In order to formalise computing the privacy properties of a complex composition of DP mechanisms such as used in DP-SGD, we define the concept of an accounting oracle that can be used to obtain the privacy bounds.
%We define an accounting oracle which gives the privacy parameters given the DP model parameters.

\begin{definition}[Accounting Oracle]\label{def:oracle}
An \emph{$\epsilon$-accounting oracle} or simply \emph{accounting oracle}, is a function $\mathcal{O}$ that evaluates $(\epsilon, \delta)$-DP privacy bounds for compositions of subsampled mechanisms. Specifically, given $\delta$, sub-sampling ratio $q$, number of iterations $T$ and a base mechanism $\mathcal{A}(\sigma)$, the oracle gives an $\epsilon$, such that a $T$-fold composition of $\mathcal{A}(\sigma)$ with sub-sampling with ratio $q$ is $(\epsilon, \delta)$-DP, \emph{i.e.},
\[\mathcal{O}:(\delta, q, T, \mathcal{A}) \mapsto \epsilon\]
\end{definition}

There are a number of approaches for realising an accounting oracle, leading to increasingly tight bounds. The advanced composition theorem \citep{dwork2014algorithmic} gives simple analytical bounds given $(\epsilon, \delta)$-DP bounds for $\mathcal{A}(\sigma)$. R\'{e}nyi DP \citep{mironov2017} is a common approach for obtaining tighter bounds, but only numerical methods \citep{koskela20b} based on privacy loss distributions \citep{sommer2019privacy} can realise an arbitrarily accurate oracle.

%\begin{example}
%We consider a dataset of size $N=20000$. We consider a continuous Gaussian mechanism $\mathcal{N}(0,\sigma^2)$ with batch size of $|B|=100$, a noise multiplier $\sigma=4$ and number of epochs $T=100$. The accounting oracle gives the values of \[\mathcal{O}:(\sigma=4, q=0.05, T=100, \mathcal{N}(0,\sigma^2)) \mapsto (\epsilon=0.71, \delta=10^{-5})\]
%\end{example}

%Differential privacy has been shown to be composable: the joint application of two $(\epsilon, \delta)$-DP mechanisms also satisfies $(\epsilon,\delta)$-DP, but with different privacy parameters reflecting the weakening of privacy guarantees through repeated use of the sensitive data.
%Accurate computation of the cumulative privacy loss parameters $(\epsilon, \delta)$ for an algorithm such as the DP-SGD is very challenging. Recently there have been many suggestions for privacy accounting, starting from the \emph{moments accountant} method proposed by \citet{Abadi2016} based on a new DP variant called R\'enyi DP \citep{Mironov2017}, which aim to analyze the privacy parameters DP-SGD. The recently introduced Fourier accountant \citep{koskela20b} uses a different approach based on privacy loss distributions \citep{sommer2019privacy} to obtain even tighter bounds on the privacy loss. Effectively, these methods can compute the cumulative privacy loss of DP-SGD as a function of the noise variance $\sigma^2$, subsampling rate $q$ and number of iterations $T$.

\subsection{Variational Inference}

In Bayesian statistics our goal is to learn the posterior distribution of parameters of a probabilistic model. For a probabilistic model $p(\dataset, \btheta)$, where $\dataset$ denotes the data and $\btheta$ the model parameters, the posterior distribution is given as
\begin{align}
    p(\btheta \mid \dataset) = \frac{p(\dataset \mid \btheta) p(\btheta)}{p(\dataset)}.
\end{align}

The exact posterior distribution is often intractable and we need to resort to approximate Bayesian inference. Variational inference (VI) approximates the intractable posterior distribution with a tractable \emph{variational posterior} $q_{\pmb{\xi}}(\pmb{\theta})$~\citep{jordan1999introduction}.
%To reduce the computational complexity, the variational posterior is often assumed to factorise over the latent variables. %The mean-field approximation assumes that the variational distribution $q_{\pmb{\xi}}(\pmb{\theta})$ is fully factorisable $q_{\pmb{\xi}}(\pmb{\theta})=\prod q_{\pmb{\xi}}(\theta_d)$.

We learn the variational posterior by minimizing the Kullback-Leibler divergence between the variational posterior $q_{\pmb{\xi}}(\pmb{\theta})$ and the true posterior $p(\btheta \mid \dataset)$:

\begin{equation}\label{eq:KL}
    KL(q_{\pmb{\xi}}(\btheta)||p(\pmb{\theta}\mid\dataset))= \E_{q_{\pmb{\xi}}(\btheta)} \left[ \log \frac{q_{\pmb{\xi}}(\btheta)}{p(\pmb{\theta}\mid \dataset))} \right]
\end{equation}

Now, this expressions again involves the exact posterior which is intractable. Therefore, instead of directly optimizing the KL divergence, VI fits the approximate posterior $q_{\pmb{\xi}}(\pmb{\theta})$ by maximizing the \emph{evidence lower bound} (ELBO) w.r.t.~the variational parameters $\bxi$. The ELBO is given as
%The optimal values ${\pmb{\xi}}^{\star}$ of the variarional parameters ${\pmb{\xi}}$ are obtained by minimising the Kullback-Leiber (KL) divergence between $p(\pmb\theta|x)$ and $q_{\pmb{\xi}}(\pmb{\theta})$, or maximising the \emph{evidence lower bound} (ELBO), such that:
\begin{align*}
    \mathcal{L}&=\int q_{\pmb{\xi}}(\pmb{\theta})\log{\frac{p(\mathcal{D},\pmb{\theta})}{q_{\pmb{\xi}}(\pmb{\theta})}} \diff\btheta \\
    &=-KL(q_{\pmb{\xi}}(\btheta)||p(\pmb{\theta})) + 
    \sum_{i=1}^{N} \E_{q_{\pmb{\xi}}(\btheta)} \left[ \log(p(\bx_i|\pmb{\theta})) \right] \\
    &=-KL(q_{\pmb{\xi}}(\btheta)||p(\pmb{\theta} \mid \dataset)) + \log p(\dataset),
\end{align*}
where $\mathcal{D}=\{\bx_1, \cdots, \bx_N\}$ is the set of observations, $\E_q$ denotes the expectation over distribution $q$ and $KL(q || p)$ the Kullback-Leibler divergence between $q$ and $p$ defined in Eq.~\eqref{eq:KL}. 

The ELBO lower bounds the marginal likelihood $\log p(\mathcal{D})$, \emph{i.e.}, it is the evidence of the marginal likelihood, and thus the name. It is straightforward to see that setting $q_{\bxi}$ to the true posterior maximizes ELBO and therefore minimizes the Kullback-divergence.

\subsubsection{Doubly stochastic variational inference}\label{sec:dsvi}

In certain models, a suitable choice of variational posterior leads to analytically tractable variational parameters. This however is not the case in general.
Doubly stochastic variational inference (DSVI) \cite{titsias2014doubly} is a framework based on stochastic gradient optimisation of the ELBO. It operates on variational approximations that lead into a differentiable form of ELBO. For such approximation, the variational parameters $\bxi$ can be learned using a stochastic gradient optimiser.
%The idea is to obtain stochastic gradients of the ELBO which can then be optimized used stochastic gradient descent (SGD) by exchanging the order of expectation and differentiation and then using a reparametrisation trick. 
For instance, samples $\btheta$ from a Gaussian variational distribution $q_{(\boldsymbol\mu, \bL)}(\btheta) = \mathcal{N}(\pmb{\theta};\pmb{\mu}, \bL \bL^T)$ can be \emph{reparametised} as 
$\btheta := \btheta(\boldsymbol\mu, \bL; \boldeta) = \boldsymbol\mu + \bL \boldeta$, where $\boldeta \sim \mathcal{N(\mathbf{0}, \mathbf{I})}$.

Denoting the parameters of the reparametrisation $\bxi := (\boldsymbol\mu, \bL)$, the sample $\btheta$ can be expressed as a stochastic function $\btheta(\bxi; \boldeta)$ which is now differentiable w.r.t.~the parameters $\boldsymbol\mu, \bL$, and thus the expression inside the expectation in ELBO is also differentiable w.r.t.~the variational parameters. This means that we can optimise the variational parameters with gradients such as 
\begin{align}
    \nabla_{\bxi} \mathcal{L} &= \nabla_{\xi} \E_{\phi(\boldeta)}
    \left[ \log \frac{p(\mathcal{D}, \btheta)}{q_{\bxi}(\btheta)} \right] \\
    %\int \nabla_{\xi} \left( q_{\bxi} (\btheta) \log \frac{p(\mathcal{D}, \btheta)}{ q_{\bxi}(\btheta)} \right)\textbf{d}\btheta \\
    &= \E_{\phi(\boldeta)}
    \left[ 
    \nabla_{\bxi} \log \frac{p(\mathcal{D},\btheta(\bxi ; \boldeta))}{q_{\bxi}(\btheta(\bxi;\boldeta))}
    \right] \\
    &= \E_{\phi(\boldeta)}
    \left[ 
    \nabla_{\bxi} \log \frac{p(\mathcal{D},\btheta(\bxi ; \boldeta)) |\det(J(\btheta))|}{\phi(\boldeta)}
    \right] \\
    &= \E_{\phi(\boldeta)}
    \left[ 
    \nabla_{\bxi}  \left( \log p(\mathcal{D},\btheta(\bxi ; \boldeta))  + \log |\det(J(\btheta))| \right)
    \right] \\
    &= \E_{\phi(\boldeta)}
    \left[ 
    \nabla_{\bxi}  \left( \log p(\mathcal{D},\btheta(\bxi ; \boldeta))  + \log |\det(\bL)| \right)
    \right],
\end{align}
where $\det(J(\btheta)) = \det(\bL)$ denotes the determinant of the Jacobian of $\theta(\bxi ; \boldeta)$ w.r.t.~the $\boldeta$, and $\phi(\boldeta)$ is the density of $\boldeta$.
%For instance, samples from a Gaussian approximation $q_{\pmb{\xi}}(\pmb{\theta}_i)=\mathcal{N}(\pmb{\theta}_i;\pmb{\mu}_i, \pmb{\Sigma}_i)$ can be written as $\pmb{\theta}_i=\pmb{\mu}_i+ \pmb{\Sigma}_i \mathbf{z}$, where $\mathbf{z} \sim \mathcal{N}(0,I)$. This would allow the use of  mini batches of the data in each iteration instead of the full dataset.

The expectations with respect to $\phi(\boldeta)$ are evaluated using Monte Carlo integration. In many cases, a single Monte Carlo sample is sufficiently accurate for SGD. Therefore, we can drop the expectation.

% The ELBO can then be written as:

% \begin{align*}
%     \mathcal{L}(q_{\pmb{\xi}})&=-KL(q_{\pmb{\xi}}(\pmb{\theta})||p(\pmb{\theta}))+ \sum_{i=1}^{B} \langle ln(p(x_i|\pmb{\theta}) \rangle_{q_{\pmb{\xi}}(\pmb{\theta})}\\
%     &=\sum_{i=1}^{B} \langle ln(p(x_i|\pmb{\theta}) \rangle_{q_{\pmb{\xi}}(\pmb{\theta})}-\frac{1}{B} KL(q_{\pmb{\xi}}(\pmb{\theta})||p(\pmb{\theta}))\\
%     &=\sum_{i=1}^{B} \mathcal{L}_i(q_{\pmb{\xi}})
% \end{align*}

To extend the approach for model parameters that are constrained, an Automatic Differentiation Variational Inference (ADVI) framework was proposed by \citet{kucukelbir2017automatic}. In this framework the constrained variables are transformed into $\mathbb{R}^d$ with a bijective map. Next the unconstrained results are given a Gaussian variational distribution, and finally the ELBO is written using the change of variables rule for the constrained variables.

% Some latent variables are constrained variables which would be hard to directly optimise. Automatic differentiation variational inference was developed by \citet{kucukelbir2017automatic}, such that it maps the constrained variables $\pmb{\theta}$ to real-valued variables $\pmb{\zeta}$. Then the variables' posterior is approximated by a Gaussian distribution and the variational inference problem is solved for the transformed variables $\pmb{\zeta}$.

\subsubsection{DP variational inference}

Differentially private variational inference (DPVI) is a technique, first proposed in \cite{JalkoHD17}, that learns the variational posteriors under differential privacy. The method is based on using DP-SGD for optimisation of the variational parameters.

\subsection{Cryptographic Techniques}

Cryptography is used in order for two or more parties to compute a function on their collective data without sharing their data with one another.

This paper will also discuss using cryptographic techniques to achieve DP. Discrete Gaussian noise will be added for the DP modelling, with privacy guarantees given in \cite{kairouz2021distributed}, which is added under MPC.

\subsubsection*{Secure Multi-party Computation}

Secure multi-party computation (MPC) \citep{yao1982protocols, maurer2006secure} is a cryptographic technique that allows parties to compute a function of their combined data without revealing the data to one another. One of the basic tools of MPC is the Shamir secret sharing scheme \citep{shamir1979share}. A $t$-out-of-$n$ secret sharing scheme assumes a dealer who wants to share a secret with $n$ parties, where any $t$ players can reconstruct the secret, however, but any subset of $t-1$ or less parties cannot reconstruct it. 

There are two general MPC protocols, \emph{Boolean MPC} and \emph{arithmetic MPC}. Boolean MPC is used to evaluate a Boolean circuit. This is based on Yao's garbled circuits \citep{yao1982protocols}. On the other hand, arithmetic MPC is used to evaluate arithmetic functions and is based on additive homomorphic encryption. Arithmetic MPC supports basic mathematical operations, such as addition and multiplication. Also, using MPC, one can approximate other arithmetic functions. However, the limitation of arithmetic MPC is that it cannot deal with floating point numbers, and therefore, a fixed point representation of real values must be used. There has been some work on MPC with floating point arithmetic such as the recent work by \cite{guo2020secure}, where approximations of functions, such as division and square roots, are derived on real numbers.

There are some general purpose software packages for MPC. In this paper, we will implement the method on CrypTen \citep{gunning2019crypten} framework. CrypTen implements arithmetic and Boolean MPC for PyTorch. In our work, we will use the arithmetic MPC implementation.

\subsection{Distributed Gaussian Noise addition}

%Otherwise, if we are assuming that there is no trusted third party in the model, the parties need to collaborate to add the DP noise. 
\citet{canonne2020discrete} showed that adding discrete Gaussian noise $\mathcal{N}_{\mathbb{Z}}(0,\sigma^2)$ provides almost identical DP guarantees as the continuous Gaussian. A discrete distributed Gaussian mechanism for noise addition was introduced by \citet{kairouz2021distributed}. It is shown that adding discrete Gaussian noise from each party can be summed up to get DP guarantees which match the non-distributed Gaussian noise addition DP. In Theorem 11 by \citet{kairouz2021distributed}, it is stated that although the convolution of two discrete Gaussians is not a discrete Gaussian, it is very close to a discrete Gaussian. This theorem is also extended to the sum of more than two Gaussians. 

\section{Privacy model}\label{sec:setup}

\subsection{Problem Setup}

For vertically partitioned data, we assume the following players in the model:

\begin{itemize}
    \item $P$ parties, \emph{i.e.}, data holders, where each party is holding a set of features for the same set of individuals.
    \item Central aggregator who is responsible for running the learning algorithm.
    \item An outside analyst who observes the data once it is published.
    \item \emph{Possible} trusted third party, \emph{e.g.} a trusted execution environment (TEE), which will add the noise.
\end{itemize}

The parties want to collaborate in order to learn a mixture model over their collective data. However, the information they own is private and cannot be shared among the parties. We will consider two options: the first is to assume each party adds a portion of the noise to the model, while the second is to assume that we have a trusted third party who will add the noise to the model before publishing the data. 

The parties will communicate with one another under MPC, such that they only share encrypted information. Lastly, the model is symmetric, \emph{i.e.}, the privacy constraints on all parties are the same.

For notational convenience, we assume that the individuals are in the same order in all the parties' datasets. We further assume that the parties are honest-but-curious, \emph{i.e.}, they would always answer honestly but would try to know more about the data from the other parties.

\subsubsection{Noise addition}

If we assume having a trusted third party (TEE), then this party will add a $\mathcal{N}_{\mathbb{Z}}(0,\sigma^2)$ discrete Gaussian noise needed in the DP mechanism. 

Otherwise, if we are assuming that there is no trusted third party in the model, the parties need to collaborate to add the DP noise. For that, each party will send its share of the discrete Gaussian noise in an encrypted form. The noise from the parties will have zero mean and $1/P$ the required noise variance, i.e. the noise of each party is sampled from $\mathcal{N}_{\mathbb{Z}}(0,\sigma^2/P)$, where $P$ is the number of parties. This noise is then summed together under MPC resulting in a sample from $\mathcal{N}_{\mathbb{Z}/P}(0,\sigma^2) \approx \mathcal{N}_{\mathbb{Z}}(0,\sigma^2)$ and added to the gradients before the decryption. The notation $\mathcal{N}_{\mathbb{Z}/P}$ is used here for a distributed discrete Gaussian formed of $P$ parts to highlight the fact that this is not identical to the plain discrete Gaussian.

\subsection{Vertically Partitioned Differential Privacy}

In the vertically partitioned data setting, each party holds a subset of the features. Therefore, with respect to each party, the data the other parties hold are secret, while the data it holds is known. However, the parties share the same individuals, and therefore contrary to the standard DP setting, the involvement of a single individual cannot be considered a secret between the parties. Thus, the privacy model of DP cannot we used when analyzing the privacy guarantees between the parties. Therefore, we need to introduce a new privacy notion. We start by defining the adjancency in the vertically partitioned setting.

\begin{definition}[Vertically Partitioned Adjacency, VP-adjacency] 
Let $\dataset$ and $\dataset'$ be datasets with $M$ features. We assume that the features are partitioned over $P$ parties $1,\cdots,P$, where party $i$ holds $m_i$ features and $\sum_i m_i=M$. Moreover, let $x_j^i\in\dataset$ be the portion of sample $x_j$ held by party $i$ and ${x_j^i}'\in\dataset'$ be the portion of sample $x_j'$ held by party $i$. We say $\dataset$ and $\dataset'$ are VP-adjacent if $\dataset = \dataset' \setminus \{x_k\} \cup \{x_k'\}$ for some $k\in[N]$, such that $x_k^p={x_k^p}'$ for some party $p\in [P]$ and $x_k^{i}\neq{x_k^{i}}'$ for all other parties $i\neq p$ such that $i\in [P]$.
\end{definition}

Now using the VP-adjacency, we define the \emph{vertically partitioned differential privacy} (VPDP):

\begin{definition}[Vertically Partitioned Differential Privacy, VPDP] 
Let $\epsilon, \delta \geq 0$. A randomised algorithm $\mathcal{A}$ is ($\epsilon,\delta$)-vertically partitioned differentially private (($\epsilon,\delta$)-VPDP) if for all pairs of VP-adjacent datasets $\dataset, \dataset'$ and all measurable subsets $S\subset \im(\mathcal{A})$
\[\Pr(\mathcal{A}(\dataset)\in S) \leq e^\epsilon \Pr(\mathcal{A}(\dataset')\in S)+\delta.\] If $\delta=0$, this is $\epsilon$-VPDP.
\end{definition}

This definition of VPDP is analogous to the notion of label differential privacy \cite{ghazi2021deep}, as in both privacy notions, a portion and not all of a sample is considered sensitive information.

\begin{thm}\label{thm:vpdp}
Any ($\epsilon,\delta$)-DP algorithm $\mathcal{A}$ %that does not rely on secrecy of subset selection
is ($\epsilon,\delta$)-VPDP.
\end{thm}

\begin{proof}
Let $\mathcal{D}$ be a dataset and let $A$ be the set of all adjacent datasets to $\mathcal{D}$ and $A_{vp}$ be the set of VP-adjacent datasets to $\mathcal{D}$. We can see that $A_{vp} \subseteq A$.

Therefore, if the output of an algorithm $\mathcal{A}$ is ($\epsilon,\delta$)-DP, \emph{i.e.}, is ($\epsilon,\delta$) indistinguishable for datasets $\mathcal{D}, \mathcal{D}'$ such that $\mathcal{D}'\subset A$, then this algorithm is also indistinguishable for datasets $\mathcal{D}, \mathcal{D}''$ such that $\mathcal{D}''\subset A_{vp}$. Therefore $\mathcal{A}$ is ($\epsilon,\delta$)-VPDP as well.

%then the output is ($\epsilon,\delta$) indistinguishable with respect to any individual sample $n$. Therefore, the output will also be ($\epsilon,\delta$) indistinguishable if part of an individual sample is changed. Therefore the mechanism is also ($\epsilon,\delta$) vertically partitioned DP.
\end{proof}

For a practical implementation of vertically partitioned DP where the parties have additional information of the internal states of the algorithm, we must be careful in applying the above theorem. The most common example of such relevant additional information is the knowledge of data elements contributing to a particular minibatch in mechanisms using sub-sampling and privacy amplification by sub-sampling, as commonly used in DP-SGD. The commonly used privacy amplification by sub-sampling requires that the contributing elements are kept secret, and hence that amplification would not apply with respect to a party who would know these.
%The sub-sampled Gaussian mechanism used in DP-SGD is $(\epsilon, \delta)$-VPDP, but the privacy parameters are in general worse because it does not enjoy privacy amplification from sub-sampling.

\section{Privacy Preserving Mixture Models over Vertically Partitioned Data}\label{sec:model}

In this section, we will discuss in detail the algorithms used to train a model on vertically partitioned data, focusing on mixture models. We will first describe the problem setting. Then, we show the techniques used for the non-partitioned case. After this, we will describe the changes needed for training the model with partitioned data. We will use MPC to combine the contribution of the data from all parties to train the model. However, we would want to minimize the use of MPC since it is computationally expensive. Therefore, we need to maximize the calculations the parties perform locally and use MPC only for combining the results.

\subsection{DPVI gradient computation}

%To use the DPVI method on vertically partitioned data (DPVI-VPD), we need to reduce the calculations needed between the parties at the cost of increasing the calculations done within each party. Therefore, we start by looking at the gradients we need to calculate.

First, we remind that, in the DPVI method, we need to optimise with respect to $\bmu$ and $\pmb L$, such that $${\pmb\theta}={\pmb L}\boldeta+\bmu.$$ Now, we look at the ELBO function:

\begin{align}
    \mathcal{L}&=\E_{q_{\bxi}(\btheta)} \left[ \log \frac{p(\dataset, \btheta)}{q_{\bxi}(\pmb\theta)}\right]\nonumber\\%\cdot \eta + \Delta_L\\
    %&=\mathbf{E} \left[ \nabla_{\pmb{\theta}} \log p(D|{\pmb{\theta}})+p({\pmb{\theta}})\right]\\%\cdot \eta + \Delta_L\\
    &=\E_{\phi(\boldeta)} \left[\log p(\dataset, \bL\boldeta+\bmu)\right]\nonumber\\
    & \quad\quad +\log|\det(\bL)|+ \mathbf{E}_{\phi(\boldeta)} \left[ -\log \phi(\eta)\right],\label{ELBO}%\cdot \eta + \Delta_L\label{ELBO}
\end{align}
where $\boldeta\sim\mathcal{N}(0,\mathbf{I})$ and $\phi(\boldeta)$ is its density.

We can then find the gradients of the ELBO with respect to $\bmu$ and $\bL$:
\begin{align}
    \nabla_{\pmb{\mu}}\mathcal{L}&=\E_{\phi(\boldeta)} \left[ \nabla_{\pmb{\theta}} \log p(\dataset, \btheta)\right]\\%\cdot \eta + \Delta_L\\
    %&=\mathbf{E} \left[ \nabla_{\pmb{\theta}} \log p(D|{\pmb{\theta}})+p({\pmb{\theta}})\right]\\%\cdot \eta + \Delta_L\\
    &=\E_{\phi(\boldeta)} \left[ \nabla_{\pmb{\theta}} \log p(\mathcal{D}|{\pmb{\theta}})\right]+\E_{\phi(\boldeta)} \left[ \nabla_{\pmb{\theta}}p({\pmb{\theta}})\right]\label{gradmu}\\\nonumber\\%\cdot \eta + \Delta_L\label{ELBO}
    \nabla_{\pmb{L}}\mathcal{L}&=\E_{\phi(\boldeta)} \left[ \nabla_{\pmb{\theta}} \log p({\pmb{\theta}},\mathcal{D})\cdot \boldeta\right] + \Delta_{\bL}\\
    %&=\mathbf{E} \left[ \nabla_{\pmb{\theta}} \log p(D|{\pmb{\theta}})+p({\pmb{\theta}})\right]\\%\cdot \eta + \Delta_L\\
    &=\E_{\phi(\boldeta)} \left[ \nabla_{\pmb{\theta}} \log p(\mathcal{D}|{\pmb{\theta}})\cdot \boldeta\right]+\E_{\phi(\boldeta)} \left[ \nabla_{\pmb{\theta}}p({\pmb{\theta}})\cdot \boldeta\right] + \Delta_{\bL},\label{gradsig}
\end{align}
where $\Delta_{\bL} = \nabla_{\bL} \log |\det(\bL)| = \mathrm{diag}(1/l_{ii})$ is the gradient of the log-determinant.

We notice that the gradient of the prior $\nabla p(\pmb\theta)$ is not private and therefore, can be calculated publicly. However, the gradient $\nabla_{\pmb{\theta}} \log p({D|\pmb{\theta}})$ in Eqs.~\eqref{gradmu} and \eqref{gradsig} needs to be found privately.

\subsection{Mixture model}

In this work, our goal is to use DPVI to learn a generative model for differentially private data sharing. Following \citet{jalko2021privacy}, we use a mixture model,
\begin{align}
    \log p(\dataset | \btheta) &= \sum_{n=1}^N \log p(\bx_n | \btheta) \\
    &= \sum_{n=1}^N \log \sum_{k=1}^K \pi_k \mathcal{F}(\bx_n | \brho_k),
    \label{eq:mixture}
\end{align}
fitted to the private data as a generative model to create a synthetic dataset.
Here the $\pi_k, k=1,\dots,K$ are the mixing weights of the $K$ components and the $\mathcal{F}(\bx_n | \brho_k)$ are the probability densities for the mixture components parametrised by $\brho_k, k=1,\dots,K$.

In order to minimise the need for communication between different parties, we assume that the density of each mixture component $\mathcal{F}(\bx_n | \brho_k)$ factorises over the $P$ different parties as
\begin{equation}
    \mathcal{F}(\bx_n | \brho_k) = \prod_{p=1}^P \mathcal{F}_p (\bx_n^p | \brho_k^p), \label{eq:densities}
\end{equation}
where $\bx_n^p$ denotes the part of the data vector $\bx_n$ held by party $p, p=1,\dots,P$ and $\brho_k^p$ denotes the parameters of $\mathcal{F}_p$.

\subsection{DPVI gradient of the mixture model}

To find the gradient $\nabla_{\pmb{\theta}} \log p(\dataset|{\pmb{\theta}})$ of the mixture model in Eq.~(\ref{eq:mixture}), we do the following: Let us assume the full dataset contains $N$ data points and $M$ features and there are $P$ parties $p=1,\cdots,P$, party $p$ having a subset of the features for all data points. We set the number of components to $K$. Now combining \eqref{eq:mixture} and \eqref{eq:densities} yields

\begin{align} \nabla_{\pmb{\theta}} \log p(\mathcal{D}|&{\pmb{\theta}})=\nabla_{\pmb{\theta}}\sum_{n=1}^N \log\left[\sum_{k=1}^K\pi_k \mathcal{F}(\bx_n|\pmb{\rho}_k)\right]\\
&=\nabla_{\pmb{\theta}}\sum_{n=1}^N \log\left[\sum_{k=1}^K\pi_k \prod_{p=1}^{P}\mathcal{F}_p(\bx^p_n|\brho_k^p)\right].\label{eq:grad}
\end{align}

%\[
%\gamma(z_{nk})=\frac{\pi_k \mathcal{F}(\mathbf{x}_n|\mathbf{\mu}_k, L_k)}{\sum_k\pi_k\mathcal{F}(\mathbf{x}_n|\mathbf{\mu}, L_k}.
%\]

From Eq.~\eqref{eq:grad}, we can see that the variables to be optimised are $\pmb\pi$ and $\pmb\rho$, therefore, we need to find the gradients with respect to those two sets of variables. Assuming $\mathcal{F}(\bx_n|\pmb{\rho}_k)$ is not dependent on $\pmb\pi$, the derivatives will be as follows.

\begin{align}
    \frac{\partial{\log p(\mathcal{D}|\btheta)}}{\partial \pi_k} &= \sum_{n=1}^N \frac{\prod_{p}\mathcal{F}_p(\bx_n^p|{\brho}_k^p)} {\sum_{k=1}^K \pi_k\prod_{p=1}^P \mathcal{F}_p(\bx_n^p|{\brho}_k^p)} \label{eq:pi}\\
    \frac{\partial{\log p(\mathcal{D}|\btheta)}}{\partial \brho_k^p} &= \sum_{n=1}^N \frac{\pi_k\frac{\partial\mathcal{F}_p({\bx}_n^p|\brho_k^p)}{\partial \brho_k^p}\cdot \prod_{p'\neq p} \mathcal{F}_{p'}({\bx}_n^{p'}|\brho^{p'}_k)}{\sum_{k=1}^K \pi_k \prod_{p=1}^P \mathcal{F}_p(\bx_n^p|\brho_k^p)}.\label{eq:rho}
\end{align}

We see from Eqs.~\eqref{eq:pi} and \eqref{eq:rho} %, that despite the simplification in Eq.~\eqref{eq:densities}, the final gradient is a sum over $N$ individuals over the $P$ parties.
that the parties can locally calculate their contributions for each individual sample and then collaborate using MPC to aggregate the information to form the final per-example gradient. Finally, DP noise is added to the sum of clipped per-example gradients, to maintain the DP guarantees. This makes the MPC problem significantly more difficult than the simple sum commonly used in secure aggregation in federated learning \citep{kairouz2019advances}.

%From Eqs.~\eqref{eq:pi} and \eqref{eq:rho}, we can see that the products in the gradient expressions are over the parties. Therefore, the parties can locally calculate their contributions and then collaborate using MPC to calculate the final products. Afterwards, the gradients would be made DP and published.

\subsection{Algorithm for DPVI-VPD for 2 parties}

\begin{algorithm} 
\SetKwData{Left}{left}\SetKwData{This}{this}\SetKwData{Up}{up}
\SetKwFunction{Union}{Union}\SetKwFunction{FindCompress}{FindCompress}
\SetKwInOut{Input}{Input}
\SetKwInOut{Output}{Output}
\SetAlgoLined
\Input{Size of the dataset $N$, Sampling ratio $q$, number of iterations $T$, clipping threshold $C$, and values at iteration $t$ ${\pmb \pi}^{(t)}, {\pmb \rho}^{(t)}$}
%\Output{}
\BlankLine
 \textbf{for} t in [T] \textbf{do}\;
 \quad Pick random subset $n$ of $N$ with sampling probability $q$\;
 \quad Send the subset $n$ along with  ${\pmb \pi}^{(t)}$ and ${\pmb \rho}^{(t)}$ to the parties\;
 \quad Each party $p$ computes its share using Eqs.~\eqref{matshare}, \eqref{dmatshre}\;
 \quad The parties use MPC to calculate the gradient ${g}_t(\bx_n)$, clip, sum and perturb it to get $\Tilde{g}_t$ then reveal the result as shown in Eqs.~\eqref{eq:gtn}-\eqref{eq:perturb}\;
 \quad Add the gradient of the log prior to $\Tilde{g}_t$ as in Eqs.~\eqref{gradmu} and \eqref{gradsig}\;
 \quad Update the values of $\pmb \pi$ and $\pmb\rho$
 \caption{DPVI for mixture models on Vertically Partitioned Data}\label{algo:main}
\end{algorithm}

In this section, we consider the case where $P=2$, \emph{i.e.}, there are two parties. 

%We notice from the equations of the gradients, Eqs.~\eqref{eq:pi} and \eqref{eq:rho}, that finding the gradients requires summing the products over the number of data points $N$, \emph{i.e.}, the individuals, which means that the summands are sensitive information.
Each party $p\in \{1,2\}$ calculates its own
$\log(\mathcal{F}_p(\bx_n^p|\brho^p_k))$ locally. Also, the party needs to compute the exponential of this for the MPC. Denote this with $\text{mat}_{n,k}^{(p)}$
\begin{align}
    \text{mat}_{n,k}^{(p)} &:= \mathcal{F}_p(\bx_n^p|\brho^p_k).
    %\exp\left(\sum_{m\in M_p}\log(\mathcal{F}_m(d_n^m|\rho^m_k))\right)\nonumber\\
    %&=\prod_{i\in M_p}\mathcal{F}_i(d_n^i|\rho_{k}^i).
    \label{matshare}
\end{align}

For the gradient update, each party needs to calculate the gradient with respect to $\brho_k^p$, thus getting the derivative \[\frac{\partial\log(\mathcal{F}_p(\bx_n^p|\brho_k^p))}{\partial \brho_{k}^p}=\frac{\frac{\partial(\mathcal{F}_p(\bx_n^p|\brho_{k}^p))}{\partial \brho_{k}^p}}{\mathcal{F}_p(\bx_n^p|\brho_{k}^p)}\] with respect to every variable $\brho_k^p$. We can denote this by 
\begin{align}
    \text{dmat}_{n,k}^{(p)}&:=\frac{\frac{\partial(\mathcal{F}_p(d_n^p|\brho_{k}^p))}{\partial \brho_{k}^p}}{\mathcal{F}_p(d_n^p|\brho_{k}^p)}.
    %\cdot \mathcal{F}_p(\bx_n^p|\rho_{k}^p) \\
    %= &\frac{\partial(\mathcal{F}_m(d_n^m|\rho_{k}^m))}{\partial \rho_{k}^m}\cdot\prod_{\substack{i \neq m \\ i\in %M_p}}\mathcal{F}_i(d_n^i|\rho_{k}^i).
    \label{dmatshre}
\end{align} 

If the DP perturbation noise $\zeta$ is added by the parties, party $p$ will generate its portion of the noise $\zeta^p=\mathcal{N}_{\mathbb{Z}}(0,C^2\sigma^2/P\mathbf{I})$. Then encrypt it and send the encrypted noise to the MPC.

After that, each party $p$ will encrypt the $\text{mat}_{n,k}^{(p)}$ and $\text{dmat}_{n,k}^{(p)}$. Using MPC, the parties collaborate to calculate the sum of their noise contribution $\pmb{\zeta}=\sum_p \pmb{\zeta}^p$ and aggregate the $\text{mat}_{n,k}^{(p)}$ and $\text{dmat}_{n,k}^{(p)}$ to obtain

\begin{align}
    \text{mat}_{n,k}&:=\text{mat}_{n,k}^{(1)}\cdot\text{mat}_{n,k}^{(2)}=\prod_{p=1}^P \mathcal{F}_p(\bx_n^p|\brho_{k}^p),\label{eq:mat}\\
    \text{dmat}_{n,k}^p&:= \pi_k \cdot \text{dmat}_{n,k}^{(p)}\cdot \text{mat}_{n,k}^{(3-p)}\nonumber\\
    &=\pi_k\frac{\partial(\mathcal{F}_p(\bx_n^p|\brho_{k}^p))}{\partial \brho_{k}^p}\cdot\prod_{p' \neq p}\mathcal{F}_{p'}(\bx_n^{p'}|\brho_{k}^{p'}).\label{eq:dmat}
\end{align}

Now, the denominator in Eqs. \eqref{eq:pi} and \eqref{eq:rho} is found using MPC as the sum \begin{equation}
    \text{den}_n:=\sum_k \pi_k \text{mat}_{n,k}.\label{eq:den}
\end{equation}

Using Eqs.~\eqref{eq:mat}, \eqref{eq:dmat} and \eqref{eq:den}, the parties calculate the individual gradients, \emph{i.e.}, the summands in Eqs. \eqref{eq:pi} and \eqref{eq:rho} under MPC respectively as follows:

\begin{align}
    GP(n,k)&=%\frac{\partial{\log p(\dataset|\btheta)}}{\partial \pi_k} &=\sum_n
    \frac{\text{mat}_{n,k}}{\text{den}_n}\label{eq:gp}\\
    GR(n,k,p)&=%\frac{\partial{\log p(\dataset|\btheta)}}{\partial \brho_k^p} &=\sum_n
    \frac{\text{dmat}_{n,k}^p}{\text{den}_n}.\label{eq:gr}
\end{align}

These per-example gradients then need to be clipped and perturbed before summing them over the individuals as in Eqs. \eqref{eq:pi} and \eqref{eq:rho}. To do that, we let 

\begin{align*}
&\mathbf{GP}(n):=\left[GP(n,1),\cdots, GP(n,K) \right],\\ &\mathbf{GR}(n,p):=\left[GR(n,1,p),\cdots, GR(n,K,p) \right],
\end{align*}

and then define the vector $g_t(\bx_n)$ as a combined vector of all the gradients individual $\bx_n$ contributes:

\begin{align}
    g_t(\bx_n)&= \left[\mathbf{GP}(n), \mathbf{GR}(n,p) \right] \label{eq:gtn}.
\end{align}

The gradients are then clipped then summed, and perturbed:

\begin{align}
    g_t^*(\bx_n)&=\frac{g_t(\bx_n)}{\max\left(1,\frac{||g_t(\bx_n)||_2}{C}\right)}\label{clip}\\
    g_t^* &= \sum_n g_t^*(\bx_n)\label{eq:sum}\\
    \Tilde{g}_t&=g_t^* + \pmb{\zeta}, \label{eq:perturb}
\end{align}
with $\pmb{\zeta} \sim \mathcal{N}_{\mathbb{Z}}(0,C^2\sigma^2\mathbf{I})$ if a trusted third party is adding the noise or $\pmb{\zeta}=\sum_p \pmb{\zeta}^p \sim \mathcal{N}_{\mathbb{Z}/P}(0,C^2\sigma^2\mathbf{I})$ if the parties are collaboratively adding the noise.

Afterwards, the gradients can be decrypted and released. Finally, we add the gradients arising from the log prior and entropy in Eqs.~\eqref{gradmu} and \eqref{gradsig} to the $\Tilde{g}_t$ to complete the gradient calculation for the update.

\subsection{MPC Challenges}

\subsubsection*{Fixed Point Precision}

We use Crypten \cite{gunning2019crypten} for multiparty computation. However, the default fixed point precision ($16$ bit precision) in Crypten is low for our purposes. The framework allows users to change the precision, which unfortunately leads to certain arithmetic and wrap around issues. Therefore, we implemented our own fixed point arithmetic which builds on the existing arithmetic available in Crypten and is based on similar techniques as in \cite{cheon2017homomorphic}. We reduce the fixed point precision to 0, and use our own implementation of multiplication, division, and exponential, using the operations built in Crypten as a base. We multiply each number by a scale $=2^{32}$ and truncate the decimal part to get only the integer part. Then, we use the following algorithms for the different calculations:

\subsubsection*{Multiplication}

Input the two encrypted values into a function mult(a,b) as in Algorithm~\ref{multalg} which divides each of the two numbers into an integer part and fraction part, and uses those parts to calculate the product.

\begin{algorithm}\caption{\texttt{mult(a,b)}}\label{multalg}
\SetKwData{Left}{left}\SetKwData{This}{this}\SetKwData{Up}{up}
\SetKwFunction{Union}{Union}\SetKwFunction{FindCompress}{FindCompress}
\SetKwInOut{Input}{Input}
\SetKwInOut{Output}{Output}
\SetAlgoLined
\Input{Two cryptensors $a$ and $b$, and $scale=2^{32}$}
%\Output{}
\BlankLine
 Find the signs $s_a=\text{sign}(a)$ and $s_b=\text{sign}(b)$\;
 Find $a_{abs}=a*s_a$ and $b_{abs}=b*s_b$\;
 Find $a_{int}=\text{truncate}(a_{abs}/scale)$ and $b_{int}=\text{truncate}(b_{abs}/scale)$\;
 Find $a_{frac}=a_{abs}-a_{int}*scale$ and $b_{frac}=b_{abs}-b_{int}*scale$\;
 Then find the output $c = s_a*s_b*(a_{int}*b_{int}*scale+a_{int}*b_{frac}+a_{frac}*b_{int}+a_{frac}*b_{frac}/scale)$
\end{algorithm}

\subsubsection*{Division and Exponential}

For the division and exponential, we basically use the same techniques used in Crypten, but we edit them slightly to use our updated multiplication and scaling the numbers up by $2^{32}$.

For division, we input the two encrypted values into a function div(a,b) as in Algorithm~\ref{divalg} which uses the Newton Raphson method, and calls the multiplication function in Algorithm~\ref{multalg}. The number of iterations $T$ in this algorithm affects the precision of the algorithm and is set by the user.

\begin{algorithm}\caption{\texttt{div(a,b)}}\label{divalg}
\SetKwData{Left}{left}\SetKwData{This}{this}\SetKwData{Up}{up}
\SetKwFunction{Union}{Union}\SetKwFunction{FindCompress}{FindCompress}
\SetKwInOut{Input}{Input}
\SetKwInOut{Output}{Output}
\SetAlgoLined
\Input{Two cryptensors $a$ and $b$, $scale=2^{32}$, and number of iterations $T$}
%\Output{}
\BlankLine
 Set an initial value for $q=3e^{scale-2*b}+(0.003*scale)\approx 1/b$\;
 For $i$ in range($T$)\;
 \quad $q=mult(q,mult(q,b))+q$\;
 Output $c=mult(q,a)$
\end{algorithm}

\subsubsection{Renormalization}\label{normcst}

\paragraph{Problem:}
In order to compute the gradients, the parties need to use MPC and compute the terms in Eqs. \eqref{eq:mat}, \eqref{eq:dmat}, and \eqref{eq:den}. We can see that the term $\text{mat}_{n,k}$ is required to compute the denominator ($den_n$) in Eqs. \eqref{eq:gp} and \eqref{eq:gr}.

%In order to compute the gradients, party $p$ will encrypt and send $\text{mat}_{n,k}^p$ as in Eq. \eqref{matshare} and \eqref{dmatshre}.
We can see in Eq. \eqref{matshare} that the value $\text{mat}_{n,k}^p$ is calculated by party $p$ by taking the exponential of the 
value $\log(\mathcal{F}_p(\bx_n^p|\brho^p_k))$. However, because of the fixed point precision, party $p$ actually has to truncate the decimal part of the scaled exponential.

\begin{equation}\label{partysend}
    mat_{n,k}^{(p)}=\text{trunc}(\exp(\log(\mathcal{F}_p(\bx_n^p|\brho^p_k)))*2^{32}).
\end{equation}

For a sufficiently small value of $\mathcal{F}_p(\bx_n^p|\brho^p_k)$, the value in \eqref{partysend} can turn out to be a zero. In such cases $\text{mat}_{n,k}=\prod_p \text{mat}_{n,k}^{(p)}=0$. This happens in the partitioned case when the log of the component density for a certain sample $n$ with respect to the features in party $p$ are very small negative numbers, \emph{i.e.}, $\log \mathcal{F}_p(\bx_n^p|\brho^p_k)<-20$. In the non-partitioned case, since their is no privacy concern, finding the gradients happens in the log-space. However, in the partitioned case, each party needs to do the exponential of this value, which will be very small that it will be truncated to zero. A bigger problem arises when this happens for every $k$ in at least one of the parties, making $\text{mat}_{n,k}=0$ for all components $k$, thus leading to $\text{den}_{n}=\sum_k\pi_k\prod_p \text{mat}_{n,k}^{(p)}=0$.

When $\text{den}_{n}=0$, the program would output a $nan$ or an $inf$. However, with the crypten arithmetics, we use the Newton Raphson method to calculate the quotient, as in Algorithm~\ref{divalg}. Therefore, in this case, the output of $div(a,b)$ where $b=0$ would be $0$. That is why the gradients in this case will end up being zero, and instead of outputting a $nan$ or $inf$, the program runs without learning.

%\begin{exmp}
%Assume there exist two parties holding a dataset on $n=10000$ individuals, where each has a single feature which can be modeled as a Gaussian variable. We further assume that the mixture model has $k=2$ mixture components. After some iterations, we get the following values from the parties:

%\begin{align*}
%    \text{mat}_{n}^{(0)}&=[0\;\;1]\\
%    \text{mat}_{n}^{(1)}&=[1\;\;0],
%\end{align*}
%for some value $n$.

%Therefore, $$\text{den}_n=\text{mat}_{n,0}^{(0)}*\text{mat}_{n,0}^{(1)}+\text{mat}_{n,1}^{(0)}*\text{mat}_{n,1}^{(1)}=0.$$
%\end{exmp}

\paragraph{Solution:}
To solve the discussed numerical issues, we will add a constant to normalize the values so they will not be truncated to $0$ when discretized. A threshold $t$ is set based on the precision and number of parties. This threshold serves to show a party if the values it outputs would likely lead the denominator to be truncated to zero.  

For each individual $n$, each party $p$ then checks:
if the value $\log(\mathcal{F}_p(\bx_n^p|\brho^p_k))<t$, the party adds a constant $c_n^{(p)}$ to it, such that $\log(\mathcal{F}_p(\bx_n^p|\brho^p_k))+c_n\geq t$. This constant will be calculated by each party as shown in Algorithm~\ref{alg:normcst}. Hence the normalizing constant from party $p$ would be an $N$-vector where $C_n^{(p)}=e^{c_n^{(p)}}$ for all $n\in[N]$. The constant will be the element-wise product of the constant from the $p$ parties such that $C=\prod_p C^{(p)}$.

%This constant will be cancelled out during the calculation of the gradient since 
For every individual record $n$, we have:

\begin{align}
    C_n\; \text{mat}_{n,k}&=\prod_p C_n^{(p)}\;\text{mat}_{n,k}^{(p)}\\
    C_n\;\text{dmat}_{n,k}^p&= \prod_p\pi_k \cdot C_n^{(p)}\text{dmat}_{n,k}^{(p)}
    \label{eq:matrices}
\end{align}

\begin{equation}C_n\;\text{den}_n=C_n\;\sum_k \pi_k \text{mat}_{n,k}.\label{eq:denom}\end{equation}

When we substitute these modified values into Eqs. \eqref{eq:gp} and \eqref{eq:gr}, we get

\begin{align*}
    GP(n,k)&=
    \frac{C_n\;\text{mat}_{n,k}}{C_n\;\text{den}_n}=
    \frac{\text{mat}_{n,k}}{\text{den}_n}\\
    GR(n,k,p)&=
    \frac{C_n\;\text{dmat}_{n,k}^p}{C_n\;\text{den}_n}=
    \frac{\text{dmat}_{n,k}^p}{\text{den}_n}.
\end{align*}

Therefore, the constant $C_n$ cancels out and we recover $GP(n,k)$ and $GR(n,k,p)$, as previously.

The added constant should take into account that we would need the product from all parties. The idea would be, most importantly, to avoid getting a zero in the denominator. Therefore, each party can check its own data and add a constant until the values are large enough that the product and sum will not be too small to be truncated by the precision. Party $p$ computes the normalizing constant $C_n^{(p)}$ as shown in Algorithm~\ref{alg:normcst}.

%\begin{enumerate}
%    \item A threshold $t$ is set based on the precision and number of parties. This threshold serves to show a party if the values it outputs could lead the denominator to be truncated to zero.
%    \item For every party, for every individual, the party checks for each individual $n$ if the value $LF_{n,k}^p:=\log \mathcal{F}_p(\bx_n^p|\brho^p_k)<t$ for all $k$.
%    \item While $LF_{n,k}^p<t$ for all $k$, the party adds a small constant $c$ to it, until at least for one value of $k=k'$, $LF_{n,k'}^p\geq t$.
%\end{enumerate}

The reason for this is that, as per the Algorithm~\ref{algo:main}, the values from a party $p$ will be multiplied with the corresponding one from the other parties and then summed for all $k$ to get the denominator, which is required to be non-zero.

We show in Figure~\ref{custom} that when a normalization constant is not added, the program does not learn and the optimisation does not happen.

\begin{algorithm}\caption{Party $p$ compute $C_n^{(p)}$}\label{alg:normcst}
\SetKwData{Left}{left}\SetKwData{This}{this}\SetKwData{Up}{up}
\SetKwFunction{Union}{Union}\SetKwFunction{FindCompress}{FindCompress}
\SetKwInOut{Input}{Input}
\SetKwInOut{Output}{Output}
\SetAlgoLined
\Input{Threshold $t$, small constant $c=0.01$}
\Output{$C_n^{(p)}$}
\BlankLine
$C_n^{(p)}=0$\;
 \texttt{while $LF_{n,k}^p:=\max_k(\log\mathcal{F}_p(\bx_n^p|\brho^p_k))<t$}\;
  \quad $\log\mathcal{F}_p(\bx_n^p|\brho^p_k)=\log\mathcal{F}_p(\bx_n^p|\brho^p_k)+c$ \;
  \quad $C_n^{(p)}=C_n^{(p)}+c$
\end{algorithm}

\section{Privacy Theorems}\label{sec:theorems}

\subsection{Privacy with respect to outside analysts}

\begin{thm}
When there is a trusted party adding the noise, Algorithm~\ref{algo:main} is $(\epsilon,\delta)$-differentially private with respect to the outside analyst, where $\epsilon = \mathcal{O}(\delta, q, T, \mathcal{N}_\mathbb{Z}(0,\sigma^2))$ is computed using the accounting oracle in Definition~\ref{def:oracle}. Here $\mathcal{N}_\mathbb{Z}(0,\sigma^2)$ denotes the discrete Gaussian mechanism with variance $\sigma^2$.
\label{thm:ext_analyst_privacy}
\end{thm}

\begin{proof}
Algorithm 1 consists of $T$ iterations of sub-sampled discrete Gaussian mechanism with sub-sampling probability $q$ and noise variance $\sigma^2$. Given $\delta$, it is $(\epsilon, \delta)$-DP with
\[ \epsilon = \mathcal{O}(\delta, q, T, \mathcal{N}_\mathbb{Z}(0,\sigma^2)) \]
computed using an accounting oracle.
\end{proof}

\begin{thm}
When the noise is added by $P$ parties collaboratively, Algorithm~\ref{algo:main} is $(\epsilon,\delta)$-differentially private with respect to the outside analyst, where $\epsilon = \mathcal{O}(\delta, q, T, \mathcal{N}_{\mathbb{Z}/P}(0,\sigma^2))$ is computed using the accounting oracle in Definition~\ref{def:oracle}. Here $\mathcal{N}_{\mathbb{Z}/P}(0,\sigma^2))$ denotes the distributed discrete Gaussian mechanism with $P$ parties and variance $\sigma^2$.
\end{thm}

The proof is analogous to that of Theorem \ref{thm:ext_analyst_privacy}.

\subsection{Privacy with respect to data holding parties}

We now turn to privacy of the algorithm relative to data holders. These parties know their own share of the data, so VPDP provides the relevant privacy model. A critical question lies with privacy amplification from sub-sampling, which requires that the elements selected for a batch need to be kept secret for the amplification to apply. 

\begin{thm}\label{thm:vp_third_party}
When a trusted third party is adding the noise, Algorithm~\ref{algo:main} is $(\epsilon,\delta)$-VPDP with respect to the $P$ parties, where $\epsilon=\mathcal{O}\left(\delta, Q, T, \mathcal{N}_\mathbb{Z}(0,\sigma^2)\right)$ is computed using the accounting oracle, and $Q=q$ if the corresponding party cannot access the sub-sample indices while $Q=1$ if they are known.
\end{thm}

\begin{proof}
Every party $p, p=1,\cdots,P$ has a subset of the sample features.
Following Theorem~\ref{thm:vpdp} and Theorem \ref{thm:ext_analyst_privacy}, we see that Algorithm~\ref{algo:main} is $(\epsilon,\delta)$-VPDP, with respect to the parties.

%Note that the privacy with respect to the parties has sub-sampling ratio $q=1$, since all parties know the data samples used in every batch.
The noise mechanism with respect to any party $p$ is  $\mathcal{A}=\mathcal{N}_\mathbb{Z}(0,\sigma^2)$.
Therefore, $\epsilon$ is computed using the accounting oracle, such that \[\epsilon=\mathcal{O}\left(\delta, 
Q, T, \mathcal{N}_\mathbb{Z}(0,\sigma^2)\right).\]

\end{proof}

This privacy bound for the case of known sub-sample indices with $Q=1$ is highly pessimistic, as it is unlikely that any sample would actually contribute to all $T$ iterations. More effective bounds in this case could be derived by foregoing random sampling of minibatches altogether and simply dividing the data set to $S$ disjoint minibatches. Running the algorithm for $E$ epochs would yield a privacy bound of
\[\epsilon=\mathcal{O}\left(\delta, 1, E, \mathcal{N}_\mathbb{Z}(0,\sigma^2)\right),\]
which would significantly improve upon Theorem \ref{thm:vp_third_party} as $E = \frac{T}{S} \approx qT$.

\begin{thm}\label{thm:vp}
When the parties add the noise collaboratively, Algorithm~\ref{algo:main} is $(\epsilon,\delta)$-VPDP with respect to the $P$ parties, where $\epsilon=\mathcal{O}\left(\delta, Q, T, \mathcal{N}_{\mathbb{Z}/P}\left(0,\frac{(P-1)}{P}\sigma^2\right)\right)$ is computed using the accounting oracle, and $Q=q$ if the corresponding party cannot access the sub-sample indices while $Q=1$ if they are known.
\end{thm}

\begin{proof}
Analogously to the proof of Theorem~\ref{thm:vp_third_party}, we can prove that Algorithm~\ref{algo:main} is $(\epsilon,\delta)$-VPDP with respect to the $P$ parties.
%with sub-sampling ratio $q=1$.
The noise mechanism with respect to a party $p$ is $\mathcal{A}=\mathcal{N}_{\mathbb{Z}/P}\left(0,\frac{(P-1)}{P}\sigma^2\right)$, as each party knows a portion of the added noise mechanism in Algorithm~\ref{algo:main}.
Therefore, $\epsilon$ is computed using the accounting oracle, such that \[\epsilon=\mathcal{O}\left(\delta, 
Q, T, \mathcal{N}_{\mathbb{Z}/P}\left(0,\frac{(P-1)}{P}\sigma^2\right)\right).\]

\end{proof}

\section{Experiments}\label{sec:exp}

%\subsection{PCA: Comparison with the non-private case}

%We have run experiments on the partitioned Adult dataset with no DP and with DP, while changing the privacy parameter $\epsilon=0.1,1,2,5,10$.

%\subsection{Fixed point and normalization stuff}

\subsection{The Adult Dataset}\label{sec:data}

Following \citet{jalko2021privacy}, we tested our model on the Adult dataset from the UCI machine learning repository \cite{Dua:2019}. In order to simulate the vertical partitioning, we divided the data between two parties where one party holds the demographic information of the individuals, and the other holds the financial information of the same individuals. We then compared the results with those of running the DPVI algorithm on the non-partitioned Adult data. 
For both cases, we model the data similarly as in \cite{jalko2021privacy}, where the mixture components are products over features modelled as follows:
\begin{itemize}
    \item We model real-valued, continuous variables, such as the age, using a Beta distribution after normalising them into $[0,1]$.
    \item For real-valued non-continuous variables, we choose to discretize  the values and model them as categorical variables.
    \item For discrete random variables with one of several options, we also use the categorical distribution.
\end{itemize}

We consider $K=20$ mixture components. The data were divided into training and testing data, where the training data constitutes $30162$ data points and the testing data $15060$ data points.
The DP noise added is chosen with zero mean and standard deviation $\sigma=2.042$ resulting to $(\epsilon,\delta)$-DP with $\epsilon=1$ and $\delta=10^{-5}$ with respect to an outside analyst. 

This also results in an $(\epsilon',\delta)$-VPDP with respect to a party $p$, $p\in[1,2]$. With subsampling ratio=1, the accounting oracle can give an upper bound and a lower bound on the values $(\epsilon',\delta)$, as in this case, the number of iterations $T$ is bounded by the number of epochs and the number of iterations. Therefore, with $\delta=10^{-5}$, $26\leq \epsilon'\leq 3000$. It can be seen here, that if DP between the parties is important, it might be best to not sub-sample, as the parties do not gain from the sub-sampling noise amplification.

For setting the priors for the model parameters, we first divide the $\brho_k$ variables into two parts. We denote the parameters for the Beta distributed features with $\tau_k$, and the parameters for the Categorical with $\omega_k$. We set the priors for $\pmb{\pi}$ and $\brho_k$ as:
\begin{align*}
    \pmb{\pi} &\sim \text{Dir}(\alpha)\\
    \omega_k &\sim \text{Dir}(\beta)\\
    \tau_k &\sim \text{Gamma}(1,1)
\end{align*}
with $\alpha = \boldsymbol{1}$ and $\beta = \boldsymbol{1}$ are vectors of all ones.

\subsection{Simulating the algorithm without Crypten}

Running the experiments with Crypten proved to be very time consuming. Therefore, we modeled Algorithm~\ref{algo:main} by using a custom fixed point implementation using the same techniques as those used in Algorithms~\ref{multalg} and \ref{divalg} but without encryption. To compare the two implementations, we assume 2 parties, each holding one feature from the partitioned described in Section~\ref{sec:data}. We then run the algorithm for 3000 iterations and batch size 100 (\emph{i.e.,} 33 epochs). We also run the same experiment without using the normalizing constant discussed in Section~\ref{normcst}.

We can see from Figure~\ref{custom} that the output from algorithm run with Crypten gave exactly the same results as without it, while the run without using the normalizing constant did not learn at all. 

Using the custom fixed point arithmetic without Crypten offers no MPC security guarantees. However, the run-time for a single iteration with Crypten was around $4.3$ seconds while that without it was around $0.09$ seconds, \emph{i.e.}, around 48 times higher run-time. Therefore, in the interest of saving on computation time, in our next experiments, we use custom fixed point implementation to show the accuracy of the results.

%\begin{figure}
%\includegraphics[scale=0.5]{DP Data sharing VPD/LL_Crypten.png}
%\caption{Negative log likelihood for optimizing with Crypten}
%\label{crypt}
%\end{figure}

\begin{figure}
\centering
\includegraphics[scale=0.55]{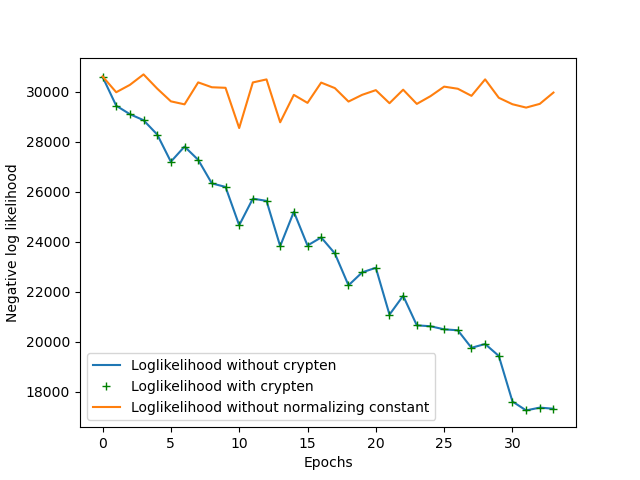}
\caption{The model learns equally well (error measured as negative log likelihood decreases) with and without the MPC (Crypten). It can be seen that without adding the normalizing constant, the model does not learn, as discussed in Section~\ref{normcst}.}
\label{custom}
\end{figure}

%\subsection{Nonpartitioned vs partitioned}

\subsection{Evaluation of Algorithm~\ref{algo:main}}

We run the DPVI algorithm used in \cite{JalkoHD17} on the non-partitioned Adult dataset and Algorithm~\ref{algo:main} on the vertically partitioned Adult data to assess the effect of fixed point arithmetics to the performance of our inference algorithm.

%To have an equivalent setting between the two algorithms, we use the same noise randomness for both DP mechanisms. 
We can expect some differences between the results from the vertically partitioned DPVI-VPD and the non-partitioned DPVI due to fundamentally different numerical representation used in the algorithm (fixed point in DPVI-VPD vs.\ floating point in the non-partitioned DPVI). Due to those differences, demonstrating exact equality between the algorithms is not possible, and thus we simply test if the two algorithms converge to equally good local optima.

We run ten repeats of Algorithm~\ref{algo:main} with $T=20000$ with minibatches of $100$ data points (i.e., the sampling ratio $q\approx 0.003$). We use the test set negative log-likelihood to measure similarity between the results from the DPVI algorithm with and without vertical partitioning. The algorithm on partitioned data took $\approx 2$-times longer than the algorithm on non-partitioned data (0.4 seconds vs. 0.2 seconds per iteration).  Figure~\ref{fig:nonpart} shows the negative log-likelihood with the number of epochs evaluated on the test data. We can see that the results are very similar, however, of course, using MPC will increase the run time, and due to fixed point arithmetic, some inaccuracies might arise.

\begin{figure}
    \centering
    \includegraphics[scale=0.55]{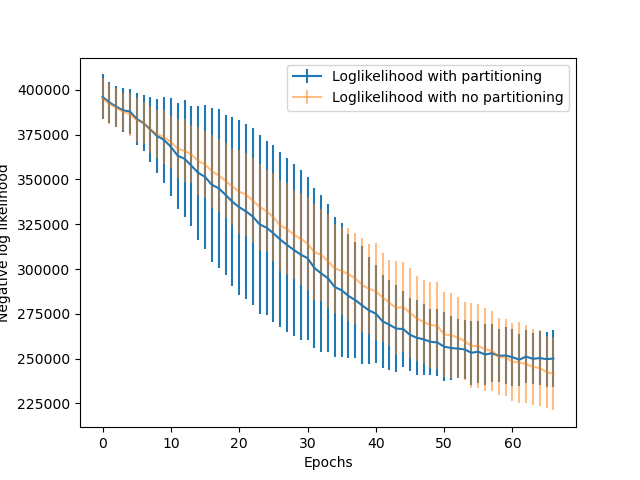}
    \caption{The model learns comparably well (error measured as test-set negative log-likelihood decreases) for DPVI on non-partitioned and vertically partitioned Adult
    data with 10 repetitions of the algorithm run on 20000 iterations, batch size = 100, and clipping threshold = 1.0. The small differences in the results are due to the discretization for MPC in the partitioned case, while using floating point arithmetics in the non-partitioned case.}
    \label{fig:nonpart}
\end{figure}

\section{Discussion}\label{sec:disc}

In this paper, we demonstrate the practical feasibility of building privacy-preserving approximate Bayesian inference on vertically partitioned data. We show results on a mixture model with mixture components that factorise across parties. This work easily generalizes to any probabilistic model where the likelihood has no dependence between the different parties.

We can also see that even if some $c\leq P-1$ parties collude to learn more about the data held by the other parties, the algorithm will still guarantee vertically partitioned DP with respect to the parties, however, the collaborating parties now know more about the data. Therefore, with $c\leq P-1$ colluding parties, we can see that Algorithm \ref{algo:main} is $(\epsilon,\delta)$-VPDP with respect to the $P$ parties, such that $\epsilon=\mathcal{O}\left(\delta, Q, T, \mathcal{N}_{\mathbb{Z}/P}\left(0,\frac{(P-c)}{P}\sigma^2\right)\right)$ with respect to the $c$ colluding parties, and $\epsilon=\mathcal{O}\left(\delta, Q, T, \mathcal{N}_{\mathbb{Z}/P}\left(0,\frac{(P-1)}{P}\sigma^2\right)\right)$ with respect to the non-colluding parties.

One important thing to note is that, if the privacy with respect to the parties is important, sub-sampling might not be a good idea or would need to be kept hidden, since otherwise there is no gain from the sub-sampling amplification.

In this work, we assume the data are matched between the different parties. This is possible if the data contains some identifiers (e.g.\ a social security number, email address or some other unique identifier) that can be used for matching, although care must be taken to restrict to the matched individuals. Assuming useful identifiers exist, one good option for matching would involve using private set intersection \citep{decristofaro2010practical} for finding the matching records and then sorting them by the identifier to establish a match. This would not reveal any additional information beyond whether a certain individual who exists in the database of one party exists in the database of the other.

If there is no single identifier, the record matching problem becomes more difficult. There exist algorithms for private entity resolution using non-unique identifying information such as names, dates of birth, etc. \citep{getoor2012entity,sehili2015privacy}. Using such algorithms increases the possibility of matching errors, which could degrade the performance of the learning. Incorporating such approach into the learning procedure might open interesting possibilities, such as using the data model to help disambiguate uncertain matches.

An obvious next question is what other models could be learned efficiently under this framework. Assuming we are willing to accept slightly larger computational cost, it might be possible to consider models requiring more communication between the parties.

\section{Conclusion}\label{sec:conc}

We have studied the problem of training a mixture model over vertically partitioned data under DP. Using differential privacy with multiparty computation proved to be non-trivial. Some extra measures were needed in order to train the DP model correctly with vertically partitioned data under MPC.

\bibliography{references}
\end{document}